\documentclass[twoside]{article}

\usepackage[accepted]{aistats2024}

\usepackage{abstract}

\usepackage[round]{natbib}

\usepackage{times}
\usepackage{helvet}
\usepackage{courier}
\usepackage{xcolor}
\usepackage{xspace}
\usepackage[linesnumbered, ruled]{algorithm2e}
\usepackage{hyperref}

\usepackage{graphicx}
\usepackage{amsmath,bm}
\usepackage{booktabs}
\usepackage{multirow}
\usepackage{amsthm}
\DeclareMathOperator{\Loss}{\mathcal{L}}
\usepackage{amssymb}
\usepackage{xcolor,colortbl}
\definecolor{maroon}{rgb}{255,218,185}
\definecolor{lavender}{rgb}{0.9, 0.9, 0.98}

\usepackage{paralist}

\theoremstyle{plain}
\newtheorem{theorem}{Theorem}[section]

\newtheorem{lemma}[theorem]{Lemma}

\theoremstyle{definition}

\theoremstyle{remark}

\begin{document}

%

%

\twocolumn[

\aistatstitle{Graph Attention-based Deep Reinforcement Learning for solving the Chinese Postman Problem with Load-dependent costs}

\aistatsauthor{Truong Son Hy \And Cong Dao Tran}

\aistatsaddress{Indiana State University, Terre Haute, IN 47809, USA \And FPT Software AI Center, Hanoi, Vietnam} ]

\begin{abstract}

Recently, Deep reinforcement learning (DRL) models have shown promising results in solving routing problems. However, most DRL solvers are commonly proposed to solve node routing problems, such as the Traveling Salesman Problem (TSP). Meanwhile, there has been limited research on applying neural methods to arc routing problems, such as the Chinese Postman Problem (CPP), since they often feature irregular and complex solution spaces compared to TSP. To fill these gaps, this paper proposes a novel DRL framework to address the CPP with load-dependent costs (CPP-LC) \citep{332b87a03ada41f3b41ad5e1c0541165}, which is a complex arc routing problem with load constraints. The novelty of our method is two-fold. First, we formulate the CPP-LC as a Markov Decision Process (MDP) sequential model. Subsequently, we introduce an autoregressive model based on DRL, namely Arc-DRL, consisting of an encoder and decoder to address the CPP-LC challenge effectively. Such a framework allows the DRL model to work efficiently and scalably to arc routing problems. Furthermore, we propose a new bio-inspired meta-heuristic solution based on Evolutionary Algorithm (EA) for CPP-LC. Extensive experiments show that Arc-DRL outperforms existing meta-heuristic methods such as Iterative Local Search (ILS) and Variable Neighborhood Search (VNS) proposed by \citep{332b87a03ada41f3b41ad5e1c0541165} on large benchmark datasets for CPP-LC regarding both solution quality and running time; while the EA gives the best solution quality with much more running time. We release our C++ implementations for metaheuristics such as EA, ILS and VNS along with the code for data generation and our generated data at \url{https://github.com/HySonLab/Chinese_Postman_Problem}.

\end{abstract}
\saythanks{Correspondence to \url{TruongSon.Hy@indstate.edu}.}
\section{Introduction} \label{sec:intro}
Routing problems are among the core NP-hard combinatorial optimization problems with a crucial role in computer science and operations research (OR). These problems have numerous real-world applications in industries such as transportation, supply chain, and operations management. As a result, they have attracted intensive research efforts from the OR community, where various exact and heuristic approaches have been proposed.
Exact methods often face limitations in terms of high time complexity, making them challenging to apply to practical situations when dealing with large-sized instances. On the other hand, traditionally, heuristic methods have been commonly used in real-world applications due to their lower computational time and complexity. However, designing a heuristic algorithm is extremely difficult because it requires specific domain knowledge in both the problem and implementation, posing significant challenges for algorithm designers when dealing with diverse combinatorial optimization problems in real-world applications.
Moreover, heuristic algorithms solve problems through trial-and-error iterative procedures, which can be time-consuming, especially for large-scale problem instances. Hence, constructing efficient data-driven methods to solve combinatorial optimization problems is essential and necessary.

Recently, machine learning methods (especially deep learning and reinforcement learning) have achieved outstanding performances in various applications in computer vision, natural language processing, speech recognition, etc. These data-driven methods also provide an alternative approach to solving combinatorial optimization problems via automatically learning complex patterns, heuristics, or policies from generated instances, thus consequentially reducing the need for well-designed human rules.
Furthermore, machine learning-based methods can accelerate the solving process and reduce computation time with certain parallel hardware platforms like GPUs, which traditional algorithms cannot take advantage of. In fact, some research has applied deep reinforcement learning (DRL) to solve routing problems, including Traveling Salesman Problems (TSP) \citep{joshi2019efficient} \citep{ijcai2021p595} \citep{hudson2022graph} and Vehicle Routing Problems (VRPs) \citep{NEURIPS2018_9fb4651c} \citep{kool2018attention} \citep{NEURIPS2020_06a9d51e} \citep{NEURIPS2021_dc9fa5f2} \citep{pmlr-v216-jiang23a}. However, TSP and VRPs are simple node routing, and many real-world routing problems are defined on edges, i.e., arc routing, and their solution is more complex with numerous constraints.

In this paper, we focus on the arc routing problem, specifically the Chinese Postman Problem with load-dependent costs (CPP-LC) \citep{332b87a03ada41f3b41ad5e1c0541165}, a complex routing problem with various practical applications in waste collection, mail delivery, and bus routing \citep{corberan2021arc}. The CPP-LC is an NP-hard combinatorial optimization problem and represents the first arc routing problem where the cost of traversing an edge depends not only on its length but also on the weight of the vehicle (including curb weight plus load) when it is traversed. Compared to TSP and VRP, the solution space of the CPP-LC can be much more irregular and challenging. To solve this problem, exact methods require significant computational time and complexity, which is only suitable for small instances. For larger instances, some conventional heuristic approaches such as Variable Neighborhood Search (VNS) and Iterated Local Search (ILS) \citep{332b87a03ada41f3b41ad5e1c0541165} have been applied but are still relatively simple, and their solutions need substantial improvement compared to the optimal solutions. Due to these existing limitations, we propose three methods to effectively tackle CPP-LC and improve existing approaches in terms of both computational time and solution quality. Specifically, our main contributions are as follows:

\begin{compactitem}

\item Firstly, we model the CPP-LC problem as a sequential solution construction problem by defining the Markov decision process model for the arc routing problem, namely Arc-MDP. This Arc-MDP model can be easily modified and applied to other arc routing problems and solved by DRL methods.
\item Secondly, we present a novel data-driven method based on Deep Reinforcement Learning and Graph Attention to solve this problem effectively and efficiently. This model can generate solutions for an instance in much less time compared to heuristic and exact methods.
\item Last but not least, we introduce two metaheuristic methods inspired by nature evolution and swarm intelligence, specifically the Evolutionary Algorithm (EA) and Ant Colony Optimization (ACO), as new metaheuristics for CPP-LC. We then conducted extensive experiments carefully to compare our Arc-DRL with several heuristic baselines on various problem instances. The empirical results demonstrate the superior performance of our approaches compared to previous methods regarding both solution quality and evaluation time. 
\end{compactitem}
\section{Related works} \label{sec:related_work}

\paragraph{Machine learning for routing problems.} \citep{bello2016neural} has introduced one of the first methods utilizing Deep Reinforcement Learning in Neural Combinatorial Optimization (DRL-NCO). Their approach is built on Pointer Network (PointerNet) \citep{vinyals2015pointer} and trained using the actor-critic method \citep{NIPS1999_6449f44a}. Structure2Vec Deep Q-learning (S2V-DQN) \citep{khalil2017learning} utilizes a graph embedding model called Structure2Vec (S2V) \citep{dai2016discriminative} as its neural network architecture to tackle certain graph-based combinatorial optimization problems by generating embeddings for each node in the input graph. Notably, on straightforward graph-based combinatorial optimization problems such as Minimum Vertex Cover, Maximum Cut, and 2D Euclidean Traveling Salesman Problem (TSP), S2V-DQN can produce solutions that closely approach the optimal solutions. Later, the methods based on Attention Model (AM) \citep{kool2018attention,kwon2020pomo} extended \citep{bello2016neural} by replacing PointerNet with the Transformer \citep{vaswani2017attention}, has become the prevailing standard Machine Learning method for NCO. AM demonstrates its problem agnosticism by effectively solving various classical routing problems and their practical extensions \citep{liao2020attention,kim2021learning}. Recent works such as \citep{jiang2022learning,bi2022learning} have also tackled the issue of impairing the cross-distribution generalization ability of neural methods for VRPs. 


\paragraph{Existing algorithms for solving CPP-LC.} 
Routing problems can be classified into two main categories: node routing, where customers are represented as nodes in a network, and arc routing problems (ARPs), where service is conducted on the arcs or edges of a network \citep{corberan2021arc}. Traditionally, research on routing problems has predominantly focused on node routing problems. However, the literature on ARPs has been steadily increasing, and the number of algorithms designed for these problems has been considerably growing in recent years. Many existing works on ARPs assume constant traversal costs for the routes, but in reality, this is often not the case in many real-life applications. One significant factor that affects the routing costs is the vehicle load. Higher loads result in increased fuel consumption and consequently higher polluting emissions. Hence, the cost of traversing an arc depends on factors like the arcs previously serviced by the route. The Chinese Postman Problem with load-dependent costs (CPP-LC) addressed this issue and was introduced by \cite{332b87a03ada41f3b41ad5e1c0541165}. In \citep{332b87a03ada41f3b41ad5e1c0541165}, it was proven that CPP-LC remains strongly NP-hard even when defined on a weighted tree. The authors proposed two formulations for this problem, one based on a pure arc-routing representation and another based on transforming the problem into a node-routing problem.  However, they could only optimally solve very small instances within a reasonable time frame. In addition, two metaheuristic algorithms, namely Iterated Local Search (ILS) and Variable Neighborhood Search (VNS) \citep{332b87a03ada41f3b41ad5e1c0541165}, have been designed to produce good solutions for larger instances. However, they are still relatively simple, and the solutions can be further improved. Therefore, ARPs with load-dependent costs pose significant challenges and deserve more attention in the future.

\section{Preliminiaries} 
\label{sec:background}
\subsection{Notations and Problem Definition of CPP-LC}
\begin{figure}
    \centering
    \includegraphics[scale=1.0]{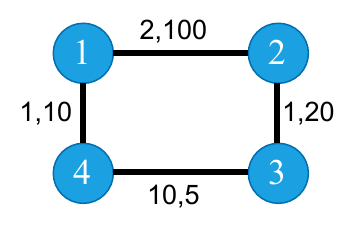}
    \caption{A CPP-LC instance.}
    \label{fig:instance}
\end{figure}
The Chinese Postman Problem with load-dependent costs (CPP-LC) can be modeled as follows. Let $G = (V, E)$ be an undirected connected graph, where $V = \{1, \ldots, n\}$ is the vertex set and
$E = \{e_1, \ldots, e_m\}$ is the edge set, where vertex $1$ represents the depot (i.e. the starting node). Each edge $e \in E$ has a
length $d_e$ and a demand of $q_e \neq 0$ units of the commodity (e.g., kilograms of salt) to be spread
on this edge. Starting at the depot, a vehicle with curb weight $W$ and loaded with $Q = \sum_{e \in E}{q_e}$ units of commodities services all domains of the graph's edges before returning to the depot. An amount $q_e$ of commodity is downloaded from the vehicle when the first time an edge $e = (i, j)$ is traversed. Any number of times in deadheading mode can be used to traverse an edge after it has been serviced. Furthermore, there is an additional cost proportional to the edge length $d_e$ multiplied by the current weight of the vehicle when each time an edge $e$ is traversed as:
\begin{math}
    d_e \times (W + \text{Load}(e)),
\end{math}
where $\text{Load}(e)$ is the load of the vehicle while traversing $e$, and it depends on the amount of commodity downloaded on the edges served before $e$ and on edge $e$ itself while serving it. Let $Q_e$ be the vehicle load at node $i$ just before traversing edge $e=(i,j)$. The load of the vehicle upon reaching node $j$ is $Q_e - q_e$, and we assume that the average load of the vehicle while serving the edge $e$ is $Q_e - \frac{q_e}{2}$.
Therefore, the load of the vehicle while traversing $e$ while serving it is $\text{Load}(e)=(Q_e - \frac{q_e}{2})$ while serving $e$. Otherwise, $\text{Load}(e)$ while traversing $e$ in deadhead is $Q_e$.

For the CPP-LC instance, a solution is any closed tour starting and beginning at the depot that traverses all the edges at least once. We need to find a CPP-LC solution with a minimum cost which is the sum of the costs associated with the traversed edges.

A simple instance of the CPP-LC problem is illustrated in Figure \ref{fig:instance}. For instance, we assume the vehicle's curb weight is $W=0$, and the length $d_e$ and the demand $q_e$ of edge $e$ correspond to two numbers next to this edge. The optimal CPP-LC for this instance is the walk $(1,2), (2,3), [3,2], [2,1], (1,4), (4,3), [3,4], [4,1]$, where $(i,j)$ means that the corresponding edge is traversed and served from $i$ to $j$ and $[i,j]$ means that the corresponding edge is deadheaded. The cost of this tour is calculated by adding the cost of all the edge traversals as follows: 

\begin{table}[htbp]
  \centering
    \begin{tabular}{rrll}
    \multicolumn{1}{c}{edge} & \multicolumn{1}{c}{vehicle load} & \multicolumn{1}{c}{cost} \\
    \multicolumn{1}{c}{$(1,2)$} & \multicolumn{1}{c}{$135$} & $2\times (0+135-\frac{100}{2})$ & $= 170$\\
    \multicolumn{1}{c}{$(2,3)$} & \multicolumn{1}{c}{$35$} & $1\times (0+35-\frac{20}{2})$ & $= 25$\\
    \multicolumn{1}{c}{$[3,2]$} & \multicolumn{1}{c}{$15$} & $1\times (0+15) $  & $= 15$\\
    \multicolumn{1}{c}{$[2,1]$} & \multicolumn{1}{c}{$15$} & $2\times (0+15) $  & $= 30$\\
    \multicolumn{1}{c}{$(1,4)$} & \multicolumn{1}{c}{$15$} &  $1 \times (0+15-\frac{10}{2})  $ & $= 10$\\
    \multicolumn{1}{c}{$(4,3)$} & \multicolumn{1}{c}{$5$} & $10 \times (0+5-\frac{5}{2}) $ & $= 25$\\
    \multicolumn{1}{c}{$[3,4]$} & \multicolumn{1}{c}{$0$} & $10\times (0+0) $ & $= 0$\\
    \multicolumn{1}{c}{$[4,1]$} & \multicolumn{1}{c}{0} & $1\times (0+0) $ & $= 0$\\ 
          &       & Total cost:  & $275$ \\
    \end{tabular}%
  \label{tab:addlabel}%
\end{table}%

We consider a simple tour $(1,2), (2,3), (3,4), (4,1)$ as another solution for this instance. Although the total length of this tour (14) is smaller than the total length of the optimal tour (28), it has a higher cost of 325.

\section{Methodology} \label{sec:methods}

In this section, we first introduce an abbreviated representation of a solution tour for CPP-LC, which can be simply represented as a sequence of edges. Then, we define an Arc routing optimization Markov decision process (Arc-MDP), and accordingly propose a new DRL framework to solve the CPP-LC problem effectively, namely Arc-DRL. 

\subsection{Abbreviated solution representation for CPP-LC}
\begin{lemma} \label{lemma:abbreviated_solution}
Given an instance of CPP-LC, a sequence of only edges to be serviced, $\sigma =((i_1,j_1),\ldots,(i_m,j_m))$ with $i_x < j_x, \forall x=1,...,m$, can represent a valid solution for the problem.
\end{lemma}

\begin{proof}
    From the sequence of the edges to be serviced, there are $2^m$ possible directions in which each edge is traversed. However, there is only one optimal direction of traversal for each edge. We can find this optimal direction by using the dynamic programming recursion, please refer to Appendix A for more details. 
\end{proof}

\begin{theorem}
    Given an instance of CPP-LC and an abbreviated solution $\sigma=((i_1,j_1),\ldots,(i_m,j_m))$, the objective value for the solution $\sigma$ on the instance can be computed in $O(m)$ time.
\end{theorem}
\begin{proof}
    The objective value of a solution $\sigma$ on the given instance can be computed by using dynamic programming as mentioned in Lemma \ref{lemma:abbreviated_solution}, i.e., $cost(\sigma)=DP(\sigma)$. The time complexity of this algorithm to find the objective value is shown to be $O(m)$. The detailed dynamic programming is shown in Appendix A.
\end{proof}

As an illustration in Figure \ref{fig:instance}, the optimal tour for the CPP-LC instance is the closed walk $(1,2), (2,3), [3,2], [2,1], (1,4), (4,3), [3,4], [4,1]$ including both edges in serving and deadheading and it seems to be complicated for directly applying metaheuristics and DRL models. However, we can represent this solution as a sequence of four edges, i.e., $\sigma=((1,2), (2,3), (1,4), (3,4))$. After that, the directions of these edges (found by using dynamic programming) as $(1,2,1), (2,3,1), (1,4,1), (3,4,2)$, where the deadheads are implied by the mismatching endnodes of traversals. Therefore, we also compute the objective value of this solution by using $DP(\sigma)=275$.

Formally, with the minimal amount of required information, a general CPP-LC tour can be represented by a vector of triplets $((i_1, j_1, d_1), (i_2,j_2,d_2), \ldots, (i_m,j_m,d_m))$, where the first two components denote the edge being serviced, and the third component denotes the direction of traversal of an edge, respectively. The direction of traversal $d=1$ means that the edge is traversed from $i$ to $j$, and $d=2$ implies that from $j$ to $i$. With this representation, the cost of a given tour can be computed in $O(m)$ time if the deadheading distances are precomputed.

\subsection{Arc routing optimization Markov decision process model}
In this section, we first define the arc routing optimization Markov decision process (Arc-MDP) as the sequential construction of a solution. Unlike node routing problems, e.g., TSP and VRP, a solution to the CPP-LC problem is represented by a closed tour of edges and must traverse all edges at least once.

For a given CPP-LC instance $\bm{P}$, we define a solution $\bm{\pi}=(\pi_1, \pi_2,\ldots,\pi_m)$ as a permutation of the edges, where $\pi_t=(i_t,j_t)$ is the edge should be serviced at time step $t$ and $\pi_t \neq \pi_{t'}, \forall t \neq t'$. The components of the corresponding Arc-MDP model are defined as follows:

1) \textit{State:} the current partial solution of CPP-LC. The state $s_t$ is the $t-$th (partially complete) solution that contains the information of the instance and current paritial solution $\bm{\pi}_{1:t-1}$. The initial and terminal states $s_0$ and $s_m$ are equivalent to the empty and completed solution, respectively. We denote the solution $\bm{\pi(P)}$ as the completed solution.

2) \textit{Action:} choose un unvisited edge. The action $a_t$ is the selection of an edge from the unvisited edges, i.e., $a_t \in \mathbb{A}_t =\{\{(i_1,j_1),(i_2,j_2),\cdot,(i_m,j_m)\} \setminus \{a_{1:t-1}\} \} $. The action space $\mathbb{A}_t$ is the set of unvisited edges.

3) \textit{Reward:} the negative value of an objective value obtained by using dynamic programming. The reward function $R(\bm{\pi}, \bm{P})$ maps the objective value from given $\bm{\pi}$ of instance $\bm{P}$.

Having defined Arc-MDP, our Arc-DRL model based on encoder-decoder defines a stochastic policy $p_\theta(\bm{\pi}_t|s_t)$. This policy represents the probability of choosing an unvisited edge $\pi_t$ as the next edge, given the current state $s_t$. Then, the policy for selecting the optimal route $\bm{\pi}$ can be factorized by chain rule as follows:

\begin{equation}
\label{eq:policy}
    p_\theta(\bm{\pi}|\bm{P}) = \prod_{t=1}^{m}p_\theta(\pi_t|s_t)
\end{equation}
where $\theta$ is the parameter of our model. 

\subsection{Graph Attention-based Deep Reinforcement Learning}

Our model, namely Arc-DRL, comprises two main components based on the attention mechanism: encoder and decoder. The encoder produces embeddings for the input graph, and the decoder produces the solution $\bm{\pi}$ of input edges, one edge at a time.

\textbf{Graph Attention Encoder.}
Normally, a Graph Attention Encoder aims at generating embeddings for the nodes on the input graph, which cannot be used directly for edges.  However, general arc routing problems, specifically the CPP-LC, inherently involve essential components of edges. Therefore, we perform a graph transformation on the CPP-LC instance prior to embedding it. Given an original graph $G(V,E)$, a new transformed graph $G'(V',E')$ is created by converting each edge $e_k=(i,j) \in E, \forall k=\{1,\ldots, |E|\}$ on the original graph to a new node $v'_k \in V'$ on the transformed graph and edges sharing a same node on the original graph becomes neighborhood nodes on the transformed graph. Notably, the depot $v_0$ on the original graph can be regarded as a virtual edge $e_0=(v_0,v_0)$ with length $d_{e_0} =0$ and demand $q_{e_0}=0$. For the transformed graph $G'(V',E')$, there are $|E|+1$ nodes, where each node $v'_i$ (i.e., edge $e_i$ on the original graph) is represented by a normalized feature $\mathbf{x}_i = (\frac{d_{e_i}}{\sum_{e\in E}d_e}, \frac{q_{e_i}}{\sum_{e\in E}q_e})$. Our encoder then generates embeddings for transformed graph $G'(V',E')$.

In general, take the $d_{\text{x}}$-dimensional input features $\mathbf{x}_i$ of node $v'_i$, the encoder uses a learned linear projection with parameters $W^{\text{x}}$ and $\mathbf{b}^{\text{x}}$ to computes initial $d_{\text{h}}$-dimensional node embeddings $\mathbf{h}_i^{(0)}$ by $\mathbf{h}_i^{(0)} = W^{\text{x}} \mathbf{x}_i + \mathbf{b}^{\text{x}}$.
In order to allow our model to distinguish the depot node from regular nodes, we additionally use separate parameters $W_0^{\text{x}}$ and $\mathbf{b}^{\text{x}}_0$ to compute the initial embedding $\mathbf{h}_0^{(0)}$ of the depot.

After obtaining the embeddings, we feed them to $N$ attention layers, each consisting of two sublayers. We denote with $\mathbf{h}_i^{(l)}$ the node embeddings produced by layer $l \in \{1, \ldots, N\}$. The encoder computes an aggregated embedding $\bar{\mathbf{h}}^{(N)}$ of the input graph as the mean of the final node embeddings $\mathbf{h}_i^{(N)}$: $\bar{\mathbf{h}}^{(N)} = \frac{1}{m} \sum_{i = 1}^m \mathbf{h}_i^{(N)}$.
Similar to AM \citep{kool2018attention},  both the node embeddings $\mathbf{h}_i^{(N)}$ and the graph embedding $\bar{\mathbf{h}}^{(N)}$ are used as input to the decoder.

For each attention layer, we follow the Transformer architecture \citep{vaswani2017attention} and Attention Model \citep{kool2018attention}. Particularly, each layer consists of two sublayers: a multi-head attention (MHA) layer that executes message passing between the nodes and a node-wise fully connected feed-forward (FF) layer. Each sublayer adds a skip connection and batch normalization as follows:

$$
    \hat{\mathbf{h}}_i = \text{BN}^l(\mathbf{h}_i^{(l-1)} + \text{MHA}^l_i(\mathbf{h}_1^{(l-1)},\ldots,\mathbf{h}_m^{(l-1)})),
$$

$$
    \mathbf{h}_i^{(l)} = \text{BN}^l(\hat{\mathbf{h}}_i + \text{FF}^l(\hat{\mathbf{h}}_i)).
$$

\textbf{Decoder.}
As introduced in our problem formulation, the decoder aims to reconstruct the solution (tour) $\bm{\pi}$ sequentially. Particularly, the decoder outputs the node $\pi_t$  at each time step $t \in \{1,\ldots,m\}$ based on the context embedding coming from the feature embeddings of the transformed graph (node and graph embeddings) obtained by the encoder and the previous outputs $\pi_t'$ at timestep $t'<t$. We add a special context for the decoder to represent context decoding. For CPP-LC, the context of the decoder at time $t$ comes from the encoder and the output up to time $t$, including the embeddings of the graph at the previous/last edge $\pi_{t-1}$, the first edge $\pi_1$ (information about the depot to starting), and the remaining vehicle load $Q_t$. For $t=1$, we use learned $d_\text{h}$-dimensional parameters $\mathbf{v}^1$ and $\mathbf{v}^f$ as input placeholders as:

$$
\mathbf{h}_{(c)}^{(N)} = 
\begin{cases}
\left[\bar{\mathbf{h}}^{(N)} , \mathbf{h}^{(N)}_{\pi_{t-1}} , \mathbf{h}^{(N)}_{\pi_1}, Q_t \right] & t > 1 \\
\left[\bar{\mathbf{h}}^{(N)} , \mathbf{v}^\text{l} , \mathbf{v}^\text{f}, Q_t\right] & t = 1.
\end{cases}
$$

Here $[\cdot,\cdot,\cdot]$ is the horizontal concatenation operator. We use $(4 \cdot d_{\text{h}})$-dimensional result vector as $\mathbf{h}_{(c)}^{(N)}$ to indicate we interpret it as the embedding of the special context node $(c)$ and use the superscript $(N)$ to align with the node embeddings $\mathbf{h}_i^{(N)}$. 
To compute a new context node embedding $\mathbf{h}_{(c)}^{(N+1)}$, we use the ($M$-head) attention mechanism described in Appendix G. The keys and values come from the node embeddings $\mathbf{h}_i^{(N)}$, but we only compute a single query $\mathbf{q}_{(c)}$ (per head) from the context node (we omit the $(N)$ for readability):
\begin{equation}
\label{eq:dec_qkv}
	\mathbf{q}_{(c)} = W^Q \mathbf{h}_{(c)} \quad \mathbf{k}_i = W^K \mathbf{h}_i, \quad \mathbf{v}_i = W^V \mathbf{h}_i.
\end{equation}

We compute the compatibility of the query with all edges and mask (set $u_{(c)j} = -\infty$) edges that cannot be visited at time $t$. For CPP-LC, this simply means we mask the edges already visited:
\begin{equation}
\label{eq:dec_compatibility}
	u_{(c)j} = \begin{cases}
		\frac{\mathbf{q}_{(c)}^T \mathbf{k}_j}{\sqrt{d_{\text{k}}}} & \text{if } j \neq \pi_{t'} \quad \forall t' < t \\
        -\infty & \text{otherwise.}
    \end{cases}
\end{equation}

Here $d_{\text{k}} = \frac{d_{\text{h}}}{M}$ is the query/key dimensionality. Again, we compute $u_{(c)j}$ and $\mathbf{v}_i$ for $M = 8$ heads and compute the final multi-head attention value for the context node using equations, but with $(c)$ instead of $i$. This mechanism is similar to our encoder but does not use skip-connections, batch normalization or the feed-forward sublayer for maximal efficiency.

To compute output probabilities $p_{\bm{\theta}}(\pi_t|s_t)$ in Equation \eqref{eq:policy}, we add one final decoder layer with a \emph{single} attention head ($M=1$ so $d_{\text{k}} = d_{\text{h}}$). For this layer, we \emph{only} compute the compatibilities $u_{(c)j}$ using \eqref{eq:dec_compatibility}, we clip the result (before masking!) within $[-C, C]$ ($C = 10$) using $\tanh$:
\begin{equation}
\label{eq:dec_logits}
	u_{(c)j} = \begin{cases}
		C \cdot \tanh \left(\frac{\mathbf{q}_{(c)}^T \mathbf{k}_j}{\sqrt{d_{\text{k}}}}\right) & \text{if } j \neq \pi_{t'} \quad \forall t' < t \\
        -\infty & \text{otherwise.}
    \end{cases}
\end{equation}
We interpret these compatibilities as unnormalized log-probabilities (logits) and compute the final output probability vector $\bm{p}$ using a softmax:

\begin{equation}
	\label{dec:probabilities}
    p_i = p_{\bm{\theta}}(\pi_t = i|s_t) = \frac{e^{u_{(c)i}}}{\sum_{j}{e^{u_{(c)j}}}}.
\end{equation}


\textbf{Optimizing Parameters with Reinforcement Learning}

\begin{algorithm}[!http]
    \textbf{Input:} number of epochs $E$, steps per epoch $T$, batch size $B$, significance $\alpha$ \\
    Init $\bm{\theta},\ \bm{\theta}^{\text{BL}} \gets \bm{\theta}$ \\
    \For{$\text{epoch} = 1, \ldots, E$}{
        \For{$\text{step} = 1, \ldots, T$}{
            $s_i \gets \text{RandomInstance()}, \forall i \in \{1, \ldots, B\}$ \\
            $\bm{\pi}_i \gets \text{SampleRollout}(s_i, p_{\bm{\theta}}), \forall i \in \{1, \ldots, B\}$ \\
            $\bm{\pi}_i^{\text{BL}} \gets \text{GreedyRollout}(s_i, p_{\bm{\theta}^{\text{BL}}}), \forall i \in \{1, \ldots, B\}$ \\
            $L(\bm{\pi}_i) \gets DP(\bm{\pi}_i)$, $L(\bm{\pi}_i^{\text{BL}}) \gets DP(\bm{\pi}_i^{\text{BL}})$, $\forall i \in \{1, \ldots, B\}$
            
            $\nabla\mathcal{L} \gets \sum_{i=1}^{B} \left(L(\bm{\pi}_i) - L(\bm{\pi}_i^{\text{BL}})\right) \nabla_{\bm{\theta}} \log p_{\bm{\theta}}(\bm{\pi}_i)$ \\
            $\bm{\theta} \gets \text{Adam}(\bm{\theta}, \nabla\mathcal{L})$
        }
        \If{$(p_{\bm{\theta}}$ is better than $p_{\bm{\theta}^{\text{BL}}})$}{
            $\bm{\theta}^{\text{BL}} \gets \bm{\theta}$
        }
    }
    \caption{Training for Arc-DRL}
    \label{algo:REINFORCE}
\end{algorithm}
As introduced in our model, given a CPP-LC instance $\bm{P}$, we can sample a solution (tour) $\bm{\pi(P)}$ from the probability distribution $p_\theta(\bm{\pi|P})$. In order to train our model, we define the loss function as the expectation of the cost $L(\bm{\pi})$: 
$$
    \Loss(\bm{\theta} | s)= \mathbb{E}_{p_{\bm{\theta}}(\bm{\pi} | s)}\left[L(\bm{\pi})\right].
$$
For CPP-LC, a solution (tour) is represented by a sequence of edges $\bm{\pi}=((i_1,j_1),\ldots,(i_m,j_m))$ and the cost (objective value) is calculated by using the dynamic programming function (as mentioned in Appendix A) as follows: $L(\bm{\pi}) = DP(\bm{\pi})$.

Then, we optimize the loss $\Loss$ using the well-known REINFORCE algorithm \citep{williams1992simple} (see Algorithm \ref{algo:REINFORCE}) to calculate the policy gradients with a baseline $b(s)$ as:
\begin{equation}
\label{eq:reinforce_baseline}
	\nabla \Loss(\bm{\theta} | s) = \mathbb{E}_{p_{\bm{\theta}}(\bm{\pi} | s)}\left[\left(L(\bm{\pi}) - b(s)\right) \nabla \log p_{\bm{\theta}}(\bm{\pi} | s)\right],
\end{equation}
where $b(s)$ is defined as the cost of a solution from a deterministic greedy rollout of the policy defined by the best model so far. If the reward of the current solution $\bm{\pi}$ is larger than the baseline, the parameters $\bm{\theta}$ of the policy network should be pushed to the direction that is more likely to generate the current solution.
We use Adam \citep{kingma2014adam} as an optimizer. The final training algorithm for Arc-DRL is shown in Algorithm \ref{algo:REINFORCE}.

\section{Experiments and Analysis} \label{sec:experiments}

In this section, we conduct experiments to evaluate the efficacy of our data-driven Arc-DRL 
model for the CPP-LC problem compared with previous heuristic approaches in terms of solution quality and running time.  


\subsection{Experimental settings}

\textbf{Baselines and Metrics.}
We compare our Arc-DRL model with the previous baselines based on heuristics, including Greedy Heuristic Construction (GHC), Iterated Local Search (ILS), and Variable Neighborhood Search (VNS), which are proposed by \citep{332b87a03ada41f3b41ad5e1c0541165} for solving CPP-LC efficiently.
Furthermore, several meta-heuristic algorithms inspired by nature evolution such as Evolutionary Algorithm (EA) \citep{muhlenbein1988evolution, prins2004simple, potvin2009state} and Ant Colony Optimization (ACO) \citep{colorni1991distributed,484436,585892,4129846} have been shown to have effective performance when dealing with combinatorial optimization problems, especially for routing problems. Therefore, we further design and implement the standard versions of these algorithms \footnote{We release our C++ implementations for EA, ACO, ILS, and VNS (including their parallel versions) along with code for data generation at \textbf{[anonymous url]}.} 
to compare with our proposed data-driven method to investigate the performance of each algorithm for the CPP-LC problem. The main structure of EA and ACO are given in the Algorithms \ref{algo:ea} and \ref{algo:ACO}, respectively. The detailed designs of our EA and ACO can be found in Appendix C and D, respectively.

\begin{algorithm}[!h]
    \textbf{Input:} Given a graph $G = (V, E)$ in which each edge $e$ is associated with demand $q_e$ and length $d_e$, $W$ as the vehicle weight, $k_{\text{max}}$ as the number of iterations, and $p_{\text{max}}$ as the maximum size of the populations. \\
    Initialize individual $\sigma^*$ using Greedy constructive heuristic (Algorithm \ref{algo:greedy})  \\
    Initialize the population $\mathcal{P} \leftarrow \{\sigma^*\}$. \\
    \For{$k = 1 \rightarrow p_{\text{max}} - 1$}{
        Perform $0.2 \times |E|$ random exchanges of edges within the sequence $\sigma^*$ to obtain $\sigma'$. \\
        $\mathcal{P} \leftarrow \mathcal{P} \cup \{\sigma'\}$
    }
    \For{$k = 1 \rightarrow k_{\text{max}}$}{
        Initialize the next generation: $\mathcal{P}' \leftarrow \emptyset$ \\
        \For{each $\sigma \in \mathcal{P}$}{
            Randomly select another $\sigma' \in \mathcal{P}$. \\
            Use Algorithm \ref{algo:mixing} to mix parents $\sigma$ and $\sigma'$ to obtain a child $\sigma''$. \\
            $\mathcal{P}' \leftarrow \mathcal{P}' \cup \{\sigma''\}$
        }
        \For{each $\sigma \in \mathcal{P}$}{
            Apply 1-OPT on $\sigma$ to obtain $\alpha$. \\
            Apply 2-OPT on $\sigma$ to obtain $\beta$. \\
            Apply 2-EXCHANGE on $\sigma$ to obtain $\gamma$. \\
            $\mathcal{P}' \leftarrow \mathcal{P}' \cup \{\alpha, \beta, \gamma\}$
        }
        Assign $\mathcal{P}$ to be the best (and unique) $p_{\text{max}}$ solutions in $\mathcal{P} \cup \mathcal{P}'$. \\
        $\sigma^* \leftarrow \text{argmin}_{\sigma \in \mathcal{P}} z(\sigma)$
    }

    \textbf{Output:} Return $\sigma^*$.
    \caption{EA for CPP-LC}
    \label{algo:ea}
\end{algorithm}

\begin{table*}[!h]
  \centering
  \caption{Results for comparison with baselines on the first dataset.}
    \scalebox{0.9}{
    \begin{tabular}{l|ccc|ccc|ccc}
    \toprule
    Dataset & \multicolumn{3}{c|}{Eulerian } & \multicolumn{3}{c|}{Christofides et al.} & \multicolumn{3}{c}{Hertz et al.} \\
    Method & Obj.  & Gap (\%) & Time (s) & Obj.  & Gap (\%) & Time (s) & Obj.  & Gap (\%) & Time (s) \\
    \midrule
    \midrule
    GHC   & 221950.80 & 17.45 & 0.009 & 89215.68 & 7.66  & 0.003 & 46487.54 & 10.02 & 0.003 \\
    ILS   & 216016.42 & 14.31 & 1.764 & 83499.37 & 0.77  & 0.493 & 44892.56 & 6.25  & 0.589 \\
    VNS   & 216468.03 & 14.55 & 0.864 & 84354.75 & 1.80  & 0.237 & 45289.50 & 7.19  & 0.275 \\
    \midrule
     \rowcolor{lavender}
    ACO   & 221938.18 & 17.44 & 0.570 & 89143.92 & 7.58  & 0.352 & 46175.24 & 9.28  & 0.390 \\
     \rowcolor{lavender}
    
    EA    & \textbf{188977.68} & \textbf{0.00} & \textbf{2.873} & \textbf{82864.59} & \textbf{0.00} & \textbf{0.816} & \textbf{42253.49} & \textbf{0.00} & \textbf{0.920} \\
    \rowcolor{lavender}
    Arc-DRL    & \textcolor[rgb]{ 0,  .09,  .941}{189064.26} & \textcolor[rgb]{ 0,  .09,  .941}{0.05} & \textcolor[rgb]{ 0,  .09,  .941}{0.725} & 
    \textcolor[rgb]{ 0,  .09,  .941}{82935.40} & \textcolor[rgb]{ 0,  .09,  .941}{0.09} & \textcolor[rgb]{ 0,  .09,  .941}{0.243} & \textcolor[rgb]{ 0,  .09,  .941}{42309.91} & \textcolor[rgb]{ 0,  .09,  .941}{0.13} & \textcolor[rgb]{ 0,  .09,  .941}{0.251} \\
    \bottomrule
    \end{tabular}%
    }
  \label{tab:res_1}%
\end{table*}%

\begin{algorithm}[!h]
    \textbf{Input:} Given a graph $G = (V, E)$ in which each edge $e$ is associated with demand $q_e$ and length $d_e$, $W$ as the vehicle weight, $k_{\text{max}}$ as the number of iterations, $p_{\text{max}}$ as the number of artificial ants, $\rho = 0.8$ as the evaporation constant, $C = 1$, and $\epsilon = 0.001$. \\
    $\tau_{xy} \leftarrow \epsilon \ \ \ (\forall x, y \in \mathcal{S})$ \\
    Initialize $\eta_{xy}$ by Eq.~(\ref{eq:init-eta-1}) and Eq.~(\ref{eq:init-eta-2}). \\
    \For{$k = 1 \rightarrow k_{\text{max}}$}{
        \For{$a = 1 \rightarrow p_{\text{max}}$}{
            Sample $\sigma_a$ by Algorithm \ref{algo:sampling}. \\
            \If{$z(\sigma_a) < z(\sigma^*)$}{
                $\sigma^* \leftarrow \sigma_a$
            }
        }
        Update the pheromone by Eq.~(\ref{eq:pheromone-update}).
    }
    \textbf{Output:} Return $\sigma^*$.
    \caption{ACO for CPP-LC}
    \label{algo:ACO}
\end{algorithm}

We use the \textit{objective value} (Obj), \textit{optimality gap} (Gap), and \textit{evaluation time} (Time) as the metrics to evaluate the performance of our model compared to other baselines. The smaller the metric values are, the better performance the model achieves.

\textbf{Datasets.}

To test the performance of our algorithm, we utilize data generated following the original study \citep{332b87a03ada41f3b41ad5e1c0541165}, divided into two sets as described in detail in Table \ref{tab:data}. The first dataset consists of 60 instances in the forms of Eulerian, Christofides et al. \citep{christofides1981algorithm}, and Hertz et al. \citep{hertz1999improvement}. For each graph, corresponding instances are then generated with $Q = \sum q_e$ and three different values of $W$, i.e., $W = 0, \frac{Q}{2}, 5Q$, where demand $q_e$ is equal to $d_e$ (proportional instance) and randomly generated (non-proportional instance). The second dataset of more difficult instances (70 instances) includes small and large types following \citep{332b87a03ada41f3b41ad5e1c0541165}. 
\begin{table}[!htb]
  \centering
  \caption{Information about two sets of instances. $|V|$ is the number of vertices, $|E|$ is the number edges, $W$ is the vehicle load, and $q_e$ is the demand serving on each edge.}
  \scalebox{0.8}{
    \begin{tabular}{clllcc}
    \toprule
          & Name  & $|V|$     & $|E|$     & $W$     &  $q_e$ \\
    \midrule
    \multirow{3}[2]{*}{Set 1} & Eulerian  & 7,10,20 & 12,18,32 & $0,\frac{Q}{2},5Q$ & $d_e$, random \\
          & Christofides & 11,14,17 & 13,32,35 & $0,\frac{Q}{2},5Q$ & $d_e$, random \\
          & Hertz & 6-27 & 11-48 & $\frac{Q}{2}$   & $d_e$, random \\
    \midrule
    \multirow{2}[2]{*}{Set 2} & small & 7,7,8 & 8,9,10 & $0,\frac{Q}{2},5Q$ & $d_e$, random \\
          & large & 10,20,30 & 16-232 & $0,\frac{Q}{2},5Q$ & $d_e$, random \\
    \bottomrule
    \end{tabular}%
    }
  \label{tab:data}%
\end{table}

\textbf{Hyperparameters and training.}
For the learning model, we trained our Arc-DRL for 100 epochs with a batch size of 512 on 100,000 instances generated on the fly. We follow the data generation in previous work \citep{332b87a03ada41f3b41ad5e1c0541165} to generate Eulerian instances with $n\in\{10,20,30\}$ and $(\alpha, k, \rho) \in \{0.2,0.33,0.5\}^3$ where node coordinates are generated randomly from $[0,1]$.

We use 3 layers in the encoder with 8 heads and set the hidden dimension to $D=128$ as the same Attention Model \citep{kool2018attention} for node routing problems. We use Adam \citep{kingma2014adam} as an optimizer and a constant learning rate as $\eta = 10^{-4}$. We implement our model in Pytorch and train on a server NVIDIA A100 40GB GPU. It takes approximately four days to finish training.

For non-learning algorithms, we program all the baselines, i.e., GHC, ILS, and VNS, as well as our new algorithms, i.e., EA and ACO, in C++ with the same environments and computing resources. For a fair comparison, we use the number of evaluations as stopping criteria and set the maximum number of evaluations as $k_{max}=100$ for all iterative algorithms (i.e., ILS, VNS, EA, and ACO). Notably, in each evaluation, a solution is evaluated to compute the cost on the given instance. For EA and ACO, we set the number of individuals/ants in a population as $p_{max}=10$.

\begin{table*}[!h]
  \centering
  \caption{Results for comparison with baselines on the second dataset.}
    \scalebox{0.9}{
    \begin{tabular}{l|c|rrr|r|rrr}
    \toprule
    Dataset & \multicolumn{4}{c|}{Small}    & \multicolumn{4}{c}{Large} \\
    Method & \multicolumn{1}{c}{} & \multicolumn{1}{c}{Obj.} & Gap (\%) & Time (s) & \multicolumn{1}{c}{} & Obj.  & Gap (\%) & Time (s) \\
    \midrule
    \midrule
    GHC   & \multirow{6}[2]{*}{\rotatebox[origin=c]{90}{\small{$|E|=8$}}} & 32046.24 & 0.00  & 0.008 & \multirow{6}[2]{*}{\rotatebox[origin=c]{90}{\tiny{$|E|=16,20,27$}}} & \multicolumn{1}{r}{200061.75} & 8.22  & 0.11 \\
    ILS   &       & 32046.24 & 0.00  & 0.121 &       & \multicolumn{1}{r}{191580.24} & 3.63  & 1.536 \\
    VNS   &       & 32046.24 & 0.00  & 0.065 &       & \multicolumn{1}{r}{193504.43} & 4.67  & 0.739 \\
    \cellcolor{lavender}ACO   &       & \cellcolor{lavender}32108.62 & \cellcolor{lavender}0.19  & \cellcolor{lavender}0.292 &       & \cellcolor{lavender}{199976.53} & \cellcolor{lavender}8.17  & \cellcolor{lavender}0.597 \\
    \cellcolor{lavender}EA    &       & \cellcolor{lavender}\textbf{32046.24} & \cellcolor{lavender}\textbf{0.00} & \cellcolor{lavender}\textbf{0.168} &       & \cellcolor{lavender}\textbf{184869.17} & \cellcolor{lavender}\textbf{0.00} & \cellcolor{lavender}\textbf{2.915} \\
    \cellcolor{lavender}Arc-DRL    &       & \cellcolor{lavender}\textcolor[rgb]{ 0,  .09,  .941}{32046.24} & \cellcolor{lavender}\textcolor[rgb]{ 0,  .09,  .941}{0.00} & \cellcolor{lavender}\textcolor[rgb]{ 0,  .09,  .941}{0.152} &       & \cellcolor{lavender}{\textcolor[rgb]{ 0,  .09,  .941}{185102.22}} & \cellcolor{lavender}\textcolor[rgb]{ 0,  .09,  .941}{0.13} & \cellcolor{lavender}\textcolor[rgb]{ 0,  .09,  .941}{0.653} \\
    \midrule
    GHC   & \multirow{6}[2]{*}{\rotatebox[origin=c]{90}{\small{$|E|=9$}}} & 24175.48 & 1.85  & 0.010 & \multirow{6}[2]{*}{\rotatebox[origin=c]{90}{\small{\tiny{$|E|=53,75$}}}} & \multicolumn{1}{r}{1823670.04} & 12.58 & 0.198 \\
    ILS   &       & 23737.03 & 0.00  & 0.164 &       & \multicolumn{1}{r}{1823670.04} & 12.58 & 40.335 \\
    VNS   &       & 23737.03 & 0.00  & {0.084} &       & \multicolumn{1}{r}{1823670.04} & 12.58 & 19.347 \\
    \cellcolor{lavender}ACO   &       & \cellcolor{lavender}23873.79 & \cellcolor{lavender}0.58  & \cellcolor{lavender}0.313 &       & \cellcolor{lavender}{1823670.04} & \cellcolor{lavender}12.58 & \cellcolor{lavender}3.215 \\
    \cellcolor{lavender}EA    &       & \cellcolor{lavender}\textbf{23737.03} & \cellcolor{lavender}\textbf{0.00} & \cellcolor{lavender}\textbf{0.247} &       & \cellcolor{lavender}{\textbf{1619820.03}} & \cellcolor{lavender}\textbf{0.00} & \cellcolor{lavender}\textbf{89.790} \\
    \cellcolor{lavender}Arc-DRL    &       & \cellcolor{lavender}\textcolor[rgb]{ 0,  .09,  .941}{23737.03} & \cellcolor{lavender}\textcolor[rgb]{ 0,  .09,  .941}{0.00} & \cellcolor{lavender}\textcolor[rgb]{ 0,  .09,  .941}{0.164} &       & \cellcolor{lavender}{\textcolor[rgb]{ 0,  .09,  .941}{1629849.83}} & \cellcolor{lavender}\textcolor[rgb]{ 0,  .09,  .941}{0.62} & \cellcolor{lavender}\textcolor[rgb]{ 0,  .09,  .941}{3.024} \\
    \midrule
    GHC   & \multirow{6}[2]{*}{\rotatebox[origin=c]{90}{\small{$|E|=10$}}} & 33723.40 & 1.54  & 0.010 & \multirow{6}[2]{*}{\rotatebox[origin=c]{90}{\tiny{$|E|=110,162,232$}}} & \multicolumn{1}{r}{12124005.61} & 8.99  & 3.494 \\
    ILS   &       & 33210.70 & 0.00  & 0.208 &       & \multicolumn{1}{r}{12124005.61} & 8.99  & 968.055 \\
    VNS   &       & 33210.70 & 0.00  & {0.112} &       & \multicolumn{1}{r}{12124005.61} & 8.99  & 368.008 \\
    \cellcolor{lavender}ACO   &       & \cellcolor{lavender}33435.88 & \cellcolor{lavender}0.68  & \cellcolor{lavender}0.340 &       & \cellcolor{lavender}{12124005.61} & \cellcolor{lavender}8.99  & \cellcolor{lavender}23.786 \\
    \cellcolor{lavender}EA    &       & \cellcolor{lavender}\textbf{33210.70} & \cellcolor{lavender}\textbf{0.00} & \cellcolor{lavender}\textbf{0.310} &       & \cellcolor{lavender}{\textbf{11123547.72}} & \cellcolor{lavender}\textbf{0.00} & \cellcolor{lavender}\textbf{3,456.023} \\
   \cellcolor{lavender}Arc-DRL    &       & \cellcolor{lavender}\textcolor[rgb]{ 0,  .09,  .941}{33210.70} & \cellcolor{lavender}\textcolor[rgb]{ 0,  .09,  .941}{0.00} & \cellcolor{lavender}\textcolor[rgb]{ 0,  .09,  .941}{0.205} &       & \cellcolor{lavender}\textcolor[rgb]{ 0,  .09,  .941}{11243163.72} & \cellcolor{lavender}\textcolor[rgb]{ 0,  .09,  .941}{1.08} & \cellcolor{lavender}\textcolor[rgb]{ 0,  .09,  .941}{10.032} \\
    \bottomrule
    \end{tabular}%
    }
  \label{tab:res_2}%
  \vspace{-0.25cm}
\end{table*}%

\subsection{Results and discussions}

\textbf{Results on dataset 1}

Table \ref{tab:res_1} presents the results obtained by our algorithms on dataset 1 with respect to the existing baselines. Respectively, the first three lines list one greedy heuristic GHC, and two strong metaheuristics, i.e., ILS and VNS proposed by \citep{332b87a03ada41f3b41ad5e1c0541165}. The following two lines are our metaheuristics inspired by biology, i.e., ACO and EA, which are conventional algorithms for solving routing problems efficiently. The last line is our learning-based method, which combines Graph Attention and Reinforcement Learning for CPP-LC, namely Arc-DRL. For the columns, column 1 indicates the ways, and columns 2-4 respectively give the average objective value, the average gap in percentage w.r.t the best solution obtained by one of these algorithms, and the running time per instance used by each algorithm on the Eulerian instances. For our learning-based model, we report the inference time of Arc-DRL. Columns 5-7 and 8-10 give the same information on the instances of Christofides et al. and Hertz et al., respectively.

As shown in Table \ref{tab:res_1}, the simplest greedy heuristic GHC always provides the fastest runtime but yields the worst results. The two metaheuristic algorithms, i.e., ILS and VNS, produce relatively good results compared to the GHC and ACO while maintaining a reasonable runtime. When compared to these baselines, our Arc-DRL model achieves superior results for all instance types, resulting in an average gap of only $0.05\%$, $0.09\%$, and $0.13\%$ on the Eulerian, Christofides, and Hertz instances respectively, when compared to the best solutions found by our EA algorithm. We observe that the Arc-DRL shows promising results by generating solutions with a relatively small gap compared to EA while still significantly reducing the runtime.

\noindent
\textbf{Results on data set 2}

We summarize in Table \ref{tab:res_2} the results obtained on the 64 instances, including small and large sizes, mentioned in Table \ref{tab:data}. As can be seen from Table \ref{tab:res_2}, on small instances with $|E|\in \{8,9,10\}$, four algorithms, i.e., ILS, VNS, EA, and RL, produce the best-found solutions for these instances. However, on large instances, only our EA among these algorithms produces the best results. Our learning method Arc-DRL obtains good solutions with a close gap to the best one (corresponding to a gap of $0.13\%, 0.62\%$, and $1.08\%$ respectively on the instances with $n=10,20, $ and 30) within a shorter time than most of the baselines except only GHC since it is deterministic and its results cannot be improved by prolonging the runtime. Furthermore, we find that on the larger instances with $|V|= 20, 30$ and $|E|\in[53,232]$, three baselines, such as ILS, VNS, and ACO, cannot improve the solution obtained by GHC and all of them correspond to a quite large gap of $12.58\%$ and $8.99\%$. 
Meanwhile, our Arc-DRL model achieves a very small gap compared to the best solution found of 1.08 \% and is 4.x-10.x faster than the previous methods VNS and ILS, and is 345.x faster than EA. One reason is that these heuristic methods are not data-driven methods and have the disadvantage of running for a long time due to many loops to solve the problem. This proves that the previous methods are inefficient on large instances and reaffirms the superiority of our Arc-DRL model in terms of solution quality and runtime.


\section{Conclusion} \label{sec:conclusion}


In this paper, we focus on investigating arc routing problems and propose data-driven neural methods to solve them. Unlike previous neural methods that are usually only applied to node routing problems, our proposed model can be applied and solved for arc routing problems with more complex solution representation and constraints, for instance, the CPP-LC problem. To do this, we define an Arc routing optimization Markov decision process (Arc-MDP) model for the arc routing problem with an abbreviated representation. Then, we propose an Arc-DRL method based on an autoregressive encoder-decoder model with attention mechanisms to solve the CPP-LC problem effectively. Furthermore, we also introduce nature-inspired algorithms (i.e., Evolutionary Algorithm (EA) and Ant Colony Optimization (ACO)) as two new meta-heuristics methods for the problem. Experimental results on many different instances show that the Arc-DRL model can achieve superior results compared to previous hand-designed heuristics. In particular, Arc-DRL achieves better results in a much shorter time than previous methods. In addition, the EA algorithm also shows promising results for solving the CPP-LC problem but has limitations in terms of long evaluation time.
We believe this research will shed new light on addressing arc routing problems and other hard combinatorial problems by data-driven methods. 

\bibliographystyle{apalike}
\bibliography{paper}

\begin{thebibliography}{}

\bibitem[Bello et~al., 2016]{bello2016neural}
Bello, I., Pham, H., Le, Q.~V., Norouzi, M., and Bengio, S. (2016).
\newblock Neural combinatorial optimization with reinforcement learning.
\newblock {\em arXiv preprint arXiv:1611.09940}.

\bibitem[Bi et~al., 2022]{bi2022learning}
Bi, J., Ma, Y., Wang, J., Cao, Z., Chen, J., Sun, Y., and Chee, Y.~M. (2022).
\newblock Learning generalizable models for vehicle routing problems via
  knowledge distillation.
\newblock {\em Advances in Neural Information Processing Systems},
  35:31226--31238.

\bibitem[Cappart et~al., 2021]{ijcai2021p595}
Cappart, Q., Chételat, D., Khalil, E.~B., Lodi, A., Morris, C., and
  Veličković, P. (2021).
\newblock Combinatorial optimization and reasoning with graph neural networks.
\newblock In Zhou, Z.-H., editor, {\em Proceedings of the Thirtieth
  International Joint Conference on Artificial Intelligence, {IJCAI-21}}, pages
  4348--4355. International Joint Conferences on Artificial Intelligence
  Organization.
\newblock Survey Track.

\bibitem[Christofides et~al., 1981]{christofides1981algorithm}
Christofides, N., Campos, V., Corber{\'a}n, A., and Mota, E. (1981).
\newblock An algorithm for the rural postman problem.
\newblock {\em Report IC. OR}, 81:81.

\bibitem[Colorni et~al., 1991]{colorni1991distributed}
Colorni, A., Dorigo, M., Maniezzo, V., et~al. (1991).
\newblock Distributed optimization by ant colonies.
\newblock In {\em Proceedings of the first European conference on artificial
  life}, volume 142, pages 134--142. Paris, France.

\bibitem[Corber{\'a}n et~al., 2021]{corberan2021arc}
Corber{\'a}n, {\'A}., Eglese, R., Hasle, G., Plana, I., and Sanchis, J.~M.
  (2021).
\newblock Arc routing problems: A review of the past, present, and future.
\newblock {\em Networks}, 77(1):88--115.

\bibitem[Corberan et~al., 2018]{332b87a03ada41f3b41ad5e1c0541165}
Corberan, A., Erdoğan, G., Laporte, G., Plana, I., and Sanchis, J. (2018).
\newblock The chinese postman problem with load-dependent costs.
\newblock {\em Transportation Science}, 52(2):370--385.

\bibitem[Dai et~al., 2016]{dai2016discriminative}
Dai, H., Dai, B., and Song, L. (2016).
\newblock Discriminative embeddings of latent variable models for structured
  data.
\newblock In {\em International conference on machine learning}, pages
  2702--2711. PMLR.

\bibitem[Delarue et~al., 2020]{NEURIPS2020_06a9d51e}
Delarue, A., Anderson, R., and Tjandraatmadja, C. (2020).
\newblock Reinforcement learning with combinatorial actions: An application to
  vehicle routing.
\newblock In Larochelle, H., Ranzato, M., Hadsell, R., Balcan, M., and Lin, H.,
  editors, {\em Advances in Neural Information Processing Systems}, volume~33,
  pages 609--620. Curran Associates, Inc.

\bibitem[Dorigo et~al., 2006]{4129846}
Dorigo, M., Birattari, M., and Stutzle, T. (2006).
\newblock Ant colony optimization.
\newblock {\em IEEE Computational Intelligence Magazine}, 1(4):28--39.

\bibitem[Dorigo and Gambardella, 1997]{585892}
Dorigo, M. and Gambardella, L. (1997).
\newblock Ant colony system: a cooperative learning approach to the traveling
  salesman problem.
\newblock {\em IEEE Transactions on Evolutionary Computation}, 1(1):53--66.

\bibitem[Dorigo et~al., 1996]{484436}
Dorigo, M., Maniezzo, V., and Colorni, A. (1996).
\newblock Ant system: optimization by a colony of cooperating agents.
\newblock {\em IEEE Transactions on Systems, Man, and Cybernetics, Part B
  (Cybernetics)}, 26(1):29--41.

\bibitem[Floyd, 1962]{10.1145/367766.368168}
Floyd, R.~W. (1962).
\newblock Algorithm 97: Shortest path.
\newblock {\em Commun. ACM}, 5(6):345.

\bibitem[Hertz et~al., 1999]{hertz1999improvement}
Hertz, A., Laporte, G., and Hugo, P.~N. (1999).
\newblock Improvement procedures for the undirected rural postman problem.
\newblock {\em INFORMS Journal on computing}, 11(1):53--62.

\bibitem[Hudson et~al., 2022]{hudson2022graph}
Hudson, B., Li, Q., Malencia, M., and Prorok, A. (2022).
\newblock Graph neural network guided local search for the traveling
  salesperson problem.
\newblock In {\em International Conference on Learning Representations}.

\bibitem[Jiang et~al., 2023]{pmlr-v216-jiang23a}
Jiang, Y., Cao, Z., Wu, Y., and Zhang, J. (2023).
\newblock Multi-view graph contrastive learning for solving vehicle routing
  problems.
\newblock In Evans, R.~J. and Shpitser, I., editors, {\em Proceedings of the
  Thirty-Ninth Conference on Uncertainty in Artificial Intelligence}, volume
  216 of {\em Proceedings of Machine Learning Research}, pages 984--994. PMLR.

\bibitem[Jiang et~al., 2022]{jiang2022learning}
Jiang, Y., Wu, Y., Cao, Z., and Zhang, J. (2022).
\newblock Learning to solve routing problems via distributionally robust
  optimization.
\newblock In {\em Proceedings of the AAAI Conference on Artificial
  Intelligence}, volume~36, pages 9786--9794.

\bibitem[Joshi et~al., 2019]{joshi2019efficient}
Joshi, C.~K., Laurent, T., and Bresson, X. (2019).
\newblock An efficient graph convolutional network technique for the travelling
  salesman problem.
\newblock {\em arXiv preprint arXiv:1906.01227}.

\bibitem[Khalil et~al., 2017]{khalil2017learning}
Khalil, E., Dai, H., Zhang, Y., Dilkina, B., and Song, L. (2017).
\newblock Learning combinatorial optimization algorithms over graphs.
\newblock {\em Advances in neural information processing systems}, 30.

\bibitem[Kim et~al., 2021]{kim2021learning}
Kim, M., Park, J., et~al. (2021).
\newblock Learning collaborative policies to solve np-hard routing problems.
\newblock {\em Advances in Neural Information Processing Systems},
  34:10418--10430.

\bibitem[Kingma and Ba, 2014]{kingma2014adam}
Kingma, D.~P. and Ba, J. (2014).
\newblock Adam: A method for stochastic optimization.
\newblock {\em arXiv preprint arXiv:1412.6980}.

\bibitem[Konda and Tsitsiklis, 1999]{NIPS1999_6449f44a}
Konda, V. and Tsitsiklis, J. (1999).
\newblock Actor-critic algorithms.
\newblock In Solla, S., Leen, T., and M\"{u}ller, K., editors, {\em Advances in
  Neural Information Processing Systems}, volume~12. MIT Press.

\bibitem[Kool et~al., 2019]{kool2018attention}
Kool, W., van Hoof, H., and Welling, M. (2019).
\newblock Attention, learn to solve routing problems!
\newblock In {\em International Conference on Learning Representations}.

\bibitem[Kwon et~al., 2020]{kwon2020pomo}
Kwon, Y.-D., Choo, J., Kim, B., Yoon, I., Gwon, Y., and Min, S. (2020).
\newblock Pomo: Policy optimization with multiple optima for reinforcement
  learning.
\newblock {\em Advances in Neural Information Processing Systems},
  33:21188--21198.

\bibitem[Li et~al., 2021]{NEURIPS2021_dc9fa5f2}
Li, S., Yan, Z., and Wu, C. (2021).
\newblock Learning to delegate for large-scale vehicle routing.
\newblock In Ranzato, M., Beygelzimer, A., Dauphin, Y., Liang, P., and Vaughan,
  J.~W., editors, {\em Advances in Neural Information Processing Systems},
  volume~34, pages 26198--26211. Curran Associates, Inc.

\bibitem[Liao et~al., 2020]{liao2020attention}
Liao, H., Dong, Q., Dong, X., Zhang, W., Zhang, W., Qi, W., Fallon, E., and
  Kara, L.~B. (2020).
\newblock Attention routing: track-assignment detailed routing using
  attention-based reinforcement learning.
\newblock In {\em International Design Engineering Technical Conferences and
  Computers and Information in Engineering Conference}, volume 84003, page
  V11AT11A002. American Society of Mechanical Engineers.

\bibitem[M{\"u}hlenbein et~al., 1988]{muhlenbein1988evolution}
M{\"u}hlenbein, H., Gorges-Schleuter, M., and Kr{\"a}mer, O. (1988).
\newblock Evolution algorithms in combinatorial optimization.
\newblock {\em Parallel computing}, 7(1):65--85.

\bibitem[Nazari et~al., 2018]{NEURIPS2018_9fb4651c}
Nazari, M., Oroojlooy, A., Snyder, L., and Takac, M. (2018).
\newblock Reinforcement learning for solving the vehicle routing problem.
\newblock In Bengio, S., Wallach, H., Larochelle, H., Grauman, K.,
  Cesa-Bianchi, N., and Garnett, R., editors, {\em Advances in Neural
  Information Processing Systems}, volume~31. Curran Associates, Inc.

\bibitem[Potvin, 2009]{potvin2009state}
Potvin, J.-Y. (2009).
\newblock State-of-the art review—evolutionary algorithms for vehicle
  routing.
\newblock {\em INFORMS Journal on computing}, 21(4):518--548.

\bibitem[Prins, 2004]{prins2004simple}
Prins, C. (2004).
\newblock A simple and effective evolutionary algorithm for the vehicle routing
  problem.
\newblock {\em Computers \& operations research}, 31(12):1985--2002.

\bibitem[Vaswani et~al., 2017]{vaswani2017attention}
Vaswani, A., Shazeer, N., Parmar, N., Uszkoreit, J., Jones, L., Gomez, A.~N.,
  Kaiser, {\L}., and Polosukhin, I. (2017).
\newblock Attention is all you need.
\newblock {\em Advances in neural information processing systems}, 30.

\bibitem[Vinyals et~al., 2015]{vinyals2015pointer}
Vinyals, O., Fortunato, M., and Jaitly, N. (2015).
\newblock Pointer networks.
\newblock {\em Advances in neural information processing systems}, 28.

\bibitem[Williams, 1992]{williams1992simple}
Williams, R.~J. (1992).
\newblock Simple statistical gradient-following algorithms for connectionist
  reinforcement learning.
\newblock {\em Machine learning}, 8:229--256.

\end{thebibliography}

\clearpage
\appendix
\onecolumn
\section{Dynamic programming to compute the cost of a CPP-LC tour} \label{sec:dp}

A CPP-LC tour is described by a vector of triplets $((i_1, j_1, d_1), (i_2, j_2, d_2), \ldots, (i_m, j_m, d_m))$, where the first two components of every triplet denote the edge being serviced, while the third component denotes the direction of an edge: $d = 1$ implies the direction from $i$ to $j$, and $d = 2$ implies the opposite direction from $j$ to $i$. Given a sequence of edges without any direction information $((i_1, j_1), (i_2, j_2), \ldots, (i_m, j_m))$, there are exactly $2^m$ possibilities of which direction each edge is traversed. Enumerating all these possibilities to find the minimum cost of a sequence of edges is computationally infeasible given a large $m$. Therefore, we apply a polynomial-time dynamic programming algorithm that is described as follows. \\

\noindent
First of all, we compute the distances of the shortest paths between all pairs of nodes in the graph by the Floyd-Warshall algorithm \citep{10.1145/367766.368168}. We donote $D_{i, j}$ as the length of the shortest path from node $i$ to node $j$, and $d_e$ or $d_{i, j}$ as the length of edge $e = (i, j)$. We also precompute the remaining amount of demand on board just before servicing the $k$-th edge $(i_k, j_k)$ in the sequence, denoted as $Q_k$. This step can be done by a simple linear-time $O(m)$ algorithm. Let $f_k(d)$ denote the minimum cost of completing the partial tour that starts from the $k$-th edge in the sequence, when the last traversal (i.e. for the $(k-1)$-th edge if $k > 1$) has been in direction $d \in \{1, 2\}$. The dynamic programming recursion is defined as follows. \\

\noindent
For the last edge:
\begin{equation}
f_m(1) = \bigg( W + \frac{q_m}{2} \bigg) d_{i_m, j_m} + \text{min} 
\begin{cases}
(W + q_m) D_{j_{m - 1}, i_m} + W D_{j_m, 1} \\
(W + q_m) D_{j_{m - 1}, j_m} + W D_{i_m, 1}
\end{cases}
\label{eq:last-edge-1}
\end{equation}

\begin{equation}
f_m(2) = \bigg( W + \frac{q_m}{2} \bigg) d_{i_m, j_m} + \text{min}
\begin{cases}
(W + q_m) D_{i_{m - 1}, i_m} + W D_{j_m, 1} \\
(W + q_m) D_{i_{m - 1}, j_m} + W D_{i_m, 1}.
\end{cases}
\label{eq:last-edge-2}
\end{equation}

\noindent
For all the middle edges ($k = 2, \ldots, m - 1$):

\begin{equation}
f_k(1) = \bigg( W + Q_k - \frac{q_k}{2} \bigg) d_{i_k, j_k} + \text{min}
\begin{cases}
(W + Q_k) D_{j_{k - 1}, i_k} + f_{k + 1}(1) \\
(W + Q_k) D_{j_{k - 1}, j_k} + f_{k + 1}(2)
\end{cases}
\label{eq:middle-edge-1}
\end{equation}

\begin{equation}
f_k(2) = \bigg( W + Q_k - \frac{q_k}{2} \bigg) d_{i_k, j_k} + \text{min}
\begin{cases}
(W + Q_k) D_{i_{k - 1}, i_k} + f_{k + 1}(1) \\
(W + Q_k) D_{i_{k - 1}, j_k} + f_{k + 1}(2).
\end{cases}
\label{eq:middle-edge-2}
\end{equation}

\noindent
For the first edge:
\begin{equation}
f_1(1) = \bigg( W + Q - \frac{q_1}{2} \bigg) d_{i_1, j_1} + \text{min}
\begin{cases}
(W + Q) d_{1, i_1} + f_2(1) \\
(W + Q) d_{1, j_1} + f_2(2).
\end{cases}
\label{eq:first-edge}
\end{equation}

\noindent
The optimum cost returned by the algorithm is $f_1(1)$. We briefly explain the construction of each equation:
\begin{itemize}
\item \textbf{For the last edge (Equations (\ref{eq:last-edge-1}) and (\ref{eq:last-edge-2})):} The load on board is only $Q_m = q_m$ left, so the first term corresponds to the cost of servicing this edge only, and must be $(W + q_m/2)d_{i_m, j_m}$ where $d_{i_m, j_m}$ denotes the edge length of $(i_m, j_m)$. For the case of $f_m(1)$ in Eq.~(\ref{eq:last-edge-1}), the previous edge has direction $d = 1$ and traverses from $i_{m - 1}$ to $j_{m - 1}$, so the last node must be $j_{m - 1}$. Here, in the second term, we have two options for the $m$-th edge:
\begin{itemize}
    \item If we traverse from $i_m$ to $j_m$ (i.e. $d = 1$) then we have to take into account the cost of traversing from $j_{m - 1}$ to $i_m$ before servicing this edge, that is $(W + q_m) D_{j_{m - 1}, i_m}$, and also the cost of traversing from $j_m$ back to the origin $1$ after servicing this edge, that is $W D_{j_m, 1}$.

    \item If we traverse from $j_m$ to $i_m$ (i.e. $d = 2$) then we have to take into account the cost of traversing from $j_{m - 1}$ to $j_m$ before servicing this edge, that is $(W + q_m) D_{j_{m - 1}, j_m}$, and also the cost of traversing from $i_m$ back to the origin $1$ after servicing this edge, that is $W D_{i_m, 1}$.
\end{itemize}
We select the option with a lower cost. Similar logic is applied to the case of $f_m(2)$ in Eq.~(\ref{eq:last-edge-2}).
\item \textbf{For the middle edges (Equations (\ref{eq:middle-edge-1}) and (\ref{eq:middle-edge-2})):} The current load is $Q_k$, so the first term corresponding to the cost of servicing this edge is $(W + Q_k - q_k/2)d_{i_k, j_k}$. The cost after servicing this edge is $f_{k + 1}(d)$ (called via recursion) where $d \in \{1, 2\}$ is the direction we traverse this edge. For the case of $f_k(1)$ with the previous node $j_{k - 1}$ in Eq.~(\ref{eq:middle-edge-1}), we again have two options to select the better one:
\begin{itemize}
    \item If we choose $d = 1$ (i.e. $i_k \rightarrow j_k$) then we have $f_{k + 1}(1)$ as the cost after servicing and $(W + Q_k) D_{j_{k - 1}, i_k}$ as the cost of traversing from $j_{k - 1}$ to $i_k$ before servicing.
    
    \item If we choose $d = 2$ (i.e. $j_k \rightarrow i_k$) then we have $f_{k + 1}(2)$ as the cost after servicing and $(W + Q_k) D_{j_{k - 1}, j_k}$ as the cost of traversing from $j_{k - 1}$ to $j_k$ before servicing.
\end{itemize}
We execute similarly for the case of $f_k(2)$.
\item \textbf{For the first edge (Equation (\ref{eq:first-edge})):} Since there is \textbf{no} other edge before this one, we assume that the ``previous'' edge had direction $d = 1$ and the previous node is always the origin $1$. Then, the logic is similar to the cases of middle edges. 
\end{itemize}
The time complexity for each case above is $O(1)$. Thus, the total time complexity is only $O(2m)$ or $O(m)$, excluding the $O(n^3)$ time for precomputing all-pairs shortest paths.

\section{Local-search operators} 

We depict three operators for local search such as:
\begin{itemize}
\item \textbf{1-OPT} (see Algorithm \ref{algo:1-OPT}): Find a way of moving an edge in the sequence to a new position so that the cost is reduced the most,
\item \textbf{2-OPT} (see Algorithm \ref{algo:2-OPT}): Find a way of reversing a subsequence of edges to minimize the cost,
\item \textbf{2-EXCHANGE} (see Algorithm \ref{algo:2-EXCHANGE}): Find the best way of swapping the positions of two edges. 
\end{itemize}

\begin{algorithm}
    \textbf{Input:} $\eta = ((i_1, j_1), (i_2, j_2), \ldots, (i_m, j_m))$. \\
    $\eta^* \leftarrow \eta$ \\
    \For{$s = 1 \rightarrow m$}{
        \For{$t = 1 \rightarrow m$}{
            Create $\eta'$ by taking the $s$-th edge out of $\eta$ and insert it into the $t$-th position (when $s \neq t$). \\
            \If{$z(\eta') < z(\eta^*)$}{
                $\eta^* \leftarrow \eta'$
            }
        }
    }
    \textbf{Output:} Return $\eta^*$.
    \caption{1-OPT}
    \label{algo:1-OPT}
\end{algorithm}

\begin{algorithm}
    \textbf{Input:} $\eta = ((i_1, j_1), (i_2, j_2), \ldots, (i_m, j_m))$. \\
    $\eta^* \leftarrow \eta$ \\
    \For{$s = 1 \rightarrow m - 1$}{
        \For{$t = s + 1 \rightarrow m$}{
            Create $\eta'$ by reversing the subsequence of edges in $\eta$ from $s$-th position to $t$-th position. \\
            We have $\eta' = ((i_1, j_1), \ldots, (i_{s - 1}, j_{s - 1}), {\color{red} (i_t, j_t), (i_{t - 1}, j_{t - 1}), \ldots, (i_s, j_s)}, (i_{t + 1}, j_{t + 1}), \ldots, (i_m, j_m)).$ \\
            \If{$z(\eta') < z(\eta^*)$}{
                $\eta^* \leftarrow \eta'$
            }
        }
    }
    \textbf{Output:} Return $\eta^*$.
    \caption{2-OPT}
    \label{algo:2-OPT}
\end{algorithm}

\begin{algorithm}
    \textbf{Input:} $\eta = ((i_1, j_1), (i_2, j_2), \ldots, (i_m, j_m))$. \\
    $\eta^* \leftarrow \eta$ \\
    \For{$s = 1 \rightarrow m - 1$}{
        \For{$t = s + 1 \rightarrow m$}{
            Create $\eta'$ by swapping the $s$-th and $t$-th edges of $\eta$. \\
            \If{$z(\eta') < z(\eta^*)$}{
                $\eta^* \leftarrow \eta'$
            }
        }
    }
    \textbf{Output:} Return $\eta^*$.
    \caption{2-EXCHANGE}
    \label{algo:2-EXCHANGE}
\end{algorithm}
\section{Evolutionary Algorithm} \label{sec:ea}

In this section, we propose an Evolutionary Algorithm (EA), a population-based metaheuristic optimization algorithm, to solve the problem CPP-LC. We define an individual as a sequence of edges (without pre-determined directions) $\sigma = ((i_1, j_1), (i_2, j_2), \ldots, (i_m, j_m))$, and its fitness (cost) can be computed by the linear-time dynamic programming (see the Appendix) as $z(\sigma)$. Inspired by biological evolution, a conventional EA under mechanisms  of reproduction: recombination, mutation, and selection. We use three local search operators, i.e., 1-OPT (see Algorithm \ref{algo:1-OPT}), 2-OPT (see Algorithm \ref{algo:2-OPT}), and 2-EXCHANGE (see Algorithm \ref{algo:2-EXCHANGE}) to mutate an individual and propose a random crossover operator (see Algorithm \ref{algo:mixing}) to recombine two parents. We then select the best/elite individuals in the population for the next generation. The algorithm has two main steps:
\begin{enumerate}
\item Generate an initial population of solutions randomly. In our case, each solution is a sequence of edges.
\item Repeat the following regenerational steps: 

\begin{enumerate}
     \item Randomly select pairs of solutions (i.e. parents) for reproduction, and ``breed'' / crossover these parents to give birth to offspring,
     \item We mutate each solution by randomly perturbing it and then applying local-search operators (e.g., 1-OPT, 2-OPT, 2-EXCHANGE),
     \item Keep the best-fit (i.e. lowest cost) in the population.
\end{enumerate}
\end{enumerate}
We describe our EA in Algorithm \ref{algo:ea}. In practice, we can significantly speedup EA by executing the crossover and mutation steps (lines from 10 to 20) in parallel / multi-threading manner (e.g., crossover and mutation for each $\sigma \in \mathcal{P}$ can be run in a separate thread).

\begin{algorithm}
    \textbf{Input:} Two sequences of $m$ edges $\sigma = \{(i_k, j_k)\}_{k = 1}^m$ and $\sigma' = \{(i'_k, j'_k)\}_{k = 1}^m$. \\
    Initialize $\sigma''$ as an empty sequence. \\
    Let $s \leftarrow 1$ and $t \leftarrow 1$. \\
    \While{$s \leq m$ and $t \leq m$}{
        Randomly select a number $c \in \{1, 2\}$. \\
        \If{$c = 1$ and $s \leq m$}{
            $e \leftarrow (i_s, j_s)$ \\
            \If{$e \notin \sigma''$}{
                Add $e$ to $\sigma''$ and set $s \leftarrow s + 1$.
            }
        }
        \If{$c = 2$ and $t \leq m$}{
            $e \leftarrow (i'_t, j'_t)$ \\
            \If{$e \notin \sigma''$}{
                Add $e$ to $\sigma''$ and set $t \leftarrow t + 1$.
            }
        }
        Break out of loop when $\text{length}(\sigma'') = m$.
    }
    If $\sigma''$ does not have enough $m$ edges, add all unselected edges left in $\sigma$ and $\sigma'$ sequentially. \\
    \textbf{Output:} Return $\sigma''$.
    \caption{Randomly mixing two sequences.}
    \label{algo:mixing}
\end{algorithm}

\section{Ant Colony Algorithm} \label{sec:aco}

Inspired by the behavior of real ants in nature, the Ant Colony Optimization algorithm (ACO) \cite{colorni1991distributed,484436,585892,4129846} is a probabilistic technique for solving combinatorial problems, in particular finding paths in a graph. ACO simulates the pheromone-based communication of biological ants by employing artificial ants (i.e. agents). Intuitively, whenever an ant finds a good path, we automatically increase the pheromone level on this path or in other words, increase the probability so that the next ants can follow. In order to construct the ACO algorithm for CPP-LC, we define the state space $\mathcal{S}$ in which each state corresponds to a tuple $(i, j, d)$ where $(i, j) \in E$ is an edge that needs to be serviced and $d \in \{1, 2\}$ denotes the direction of servicing on this edge. There are exactly $2|E|$ states. Each ant needs to construct a solution or a sequence of $|E|$ edges/states. We also have an additional state $s^*$ indicating that the ant is still at the starting node (i.e. depot) 1 and has not ``visited'' any other states. At each step of the algorithm, each ant moves from state $x$ to state $y$ with the probability $p_{xy}$ that depends on the attractiveness $\eta_{xy}$ of the move, as computed by some heuristic, and the pheromone level $\tau_{xy}$ of the move, that is proportional to the number of times previous ants chose this path in the past and its quality:
$$
    p(y|x) = \frac{\tau_{xy} \cdot \eta_{xy}}{\sum_{z \in N(x)} \tau_{xz} \cdot \eta_{xz}},
$$
where $N(x)$ denotes the set of feasible (i.e. not-visited) states to visit from $x$. 
Pheromone levels are updated when all ants have completed their path. Depending on the quality of the solutions, the pheromone levels are adjusted (i.e. increased or decreased) accordingly:
\begin{equation}
\tau_{xy} \leftarrow (1 - \rho) \tau_{xy} + \sum_{a = 1}^{p_{\text{max}}} \Delta \tau^a_{xy},
\label{eq:pheromone-update}
\end{equation}
where $\rho$ is the pheromone evaporation coefficient, $p_{\text{max}}$ is the number of ants, and $\Delta \tau^a_{xy}$ is the amount of pheromone deposited by the $a$-th ant that is defined as:
$$
\Delta \tau^a_{xy} = 
\begin{cases}
C / \sqrt{z(\sigma_a)} & \text{if this ant transits from state $x$ to $y$} \\
0 & \text{otherwise}
\end{cases}
$$
where $\sigma_a$ denotes the sequence of edges / states sampled by the $a$-th ant, $z(\sigma_a)$ as its cost, and $C$ is a constant. We describe our ACO algorithm to solve CPP-LC in Algorithm \ref{algo:ACO}. In practice, we can utilize parallelism / multi-threading to speedup by running the sampling procedures (lines 7 to 12) on multiple threads. 

We initialize the priori $\eta$ from the state $x = (i, j, d)$ as follows:
\begin{equation}
\eta_{xy} = \sqrt{(W + q_e) D_{j, i'} + (W + q_e / 2) d_e},
\label{eq:init-eta-1}
\end{equation}
where state $y = (i', j', d')$ corresponds to the edge $e = (i', j')$, and $D$ is a matrix containing the lengths of all-pair shortest paths. In the case $x = s^*$, instead of Eq.~(\ref{eq:init-eta-1}), we have:
\begin{equation}
\eta_{s^*y} = \sqrt{(W + Q) D_{1, i'} + (W + Q - q_e / 2) d_e}.
\label{eq:init-eta-2}
\end{equation}
Algorithm \ref{algo:sampling} describes the sampling procedure of Ant Colony Optimization (ACO) for CPP-LC.

\begin{algorithm}
    \textbf{Input:} $\tau_{xy}$, $\eta_{xy}$ and the starting state $s^*$. \\
    Initialize the set of available edges $\bar{E} \leftarrow E$. \\
    $s \leftarrow s^*$ \\
    $\sigma \leftarrow \emptyset$ \\
    \For{$i = 1 \rightarrow |E|$}{
        Construct the set of feasible states $N(s) = \{t = (i, j, d) | (i, j) \in \bar{E}\}$. \\
        Compute the transition probability:
        $$p(s'|s) = \frac{\tau_{ss'} \cdot \eta_{ss'}}{\sum_{t \in N(s)} \tau_{st} \cdot \eta_{st}}.$$ \\
        Sample $s'$ from $p(s'|s)$. Assume $s' = (i, j, d)$. \\
        $s \leftarrow s'$ \\
        $\sigma \leftarrow \sigma \cup (i, j, d)$ \\
        $\bar{E} \leftarrow \bar{E} \setminus (i, j)$
    }
    \textbf{Output:} Return $\sigma$.
    \caption{Sampling procedure of ACO for CPP-LC}
    \label{algo:sampling}
\end{algorithm}

\section{Greedy Heuristic Construction} \label{sec:greedy}

In this section, we introduce the Greedy Heuristic Construction (GHC) in Algorithm \ref{algo:greedy} originally proposed by \cite{332b87a03ada41f3b41ad5e1c0541165}. We denote a sequence of edges (without pre-determined directions) as $\sigma = ((i_1, j_1), (i_2, j_2), \ldots, (i_m, j_m))$, and its optimal cost computed by the linear-time dynamic programming (see the previous section) as $z(\sigma)$.

\begin{algorithm}
    \textbf{Input:} We are given a graph $G = (V, E)$ in which each edge $e$ is associated with demand $q_e$ and length $d_e$, and $W$ as the vehicle weight. \\
    Sort the set of edges $E$ in decreasing order of demand multiplied by distance, i.e. $d_e \times q_e$. \\
    $\sigma^* \leftarrow \emptyset$ \\
    \For{$i = 1 \rightarrow |E|$}{
        $z_{\text{min}} \leftarrow \infty$ \\
        \For{$j = 1 \rightarrow i$}{
            Create $\sigma'$ by inserting the $i$-th edge in the list in the $j$-th position of $\sigma^*$. \\
            \If{$z(\sigma') < z_{\text{min}}$}{
                $z_{\text{min}} \leftarrow z(\sigma')$ \\
                $\sigma'' \leftarrow \sigma'$
            }
        }
        $\sigma^* \leftarrow \sigma''$
    }
    \textbf{Output:} Return $\sigma^*$.
    \caption{Greedy heuristic construction \cite{332b87a03ada41f3b41ad5e1c0541165}}
    \label{algo:greedy}
\end{algorithm}

\section{Iterated Local Search} \label{sec:ils}

We introduce Iterated Local Search (ILS) metaheuristic in Algorithm \ref{algo:ils}. The key idea of ILS is to randomly perturb a solution and then apply the local search including 1-OPT, 2-OPT and 2-EXCHANGE operators to improve it. The algorithm is repeated with $k_{\text{max}}$ iterations. In practice, we implement a parallel version of ILS in which in each iteration, all local-search operators run simultaneously (i.e. each operator runs on a different thread). This multi-threading technique improves the computational time.

\begin{algorithm}
    \textbf{Input:} We are given a graph $G = (V, E)$ in which each edge $e$ is associated with demand $q_e$ and length $d_e$, $W$ as the vehicle weight, and $k_{\text{max}}$ as the number of iterations. \\
    Call the Greedy constructive heuristic (Algorithm \ref{algo:greedy}) to initialize $\sigma^*$. \\
    $\sigma' \leftarrow \sigma^*$ \\
    \For{$k = 1 \rightarrow k_{\text{max}}$}{
        Perform $0.2 \times |E|$ random exchanges of edges within the sequence $\sigma'$. \\
        Apply local search on $\sigma'$ by selecting the best move among the operators 1-OPT, 2-OPT and 2-EXCHANGE. \\
        \eIf{$z(\sigma') < z(\sigma^*)$}{
            $\sigma^* \leftarrow \sigma'$
        }{
            $\sigma' \leftarrow \sigma^*$
        }
    }
    \textbf{Output:} Return $\sigma^*$.
    \caption{Iterated Local Search (ILS) metaheuristic \cite{332b87a03ada41f3b41ad5e1c0541165}}
    \label{algo:ils}
\end{algorithm}

\section{Variable Neighborhood Search} \label{sec:vns}

We present Variable Neighborhood Search (VNS) in Algorithm \ref{algo:vns}. Different from ILS, we use a specific order of operators: 2-EXCHANGE, 1-OPT and finally 2-OPT. This order is increasing in terms of computational cost. We stop the local search immediately after an operator can improve the current best solution. In practice, VNS is generally faster than ILS.

\begin{algorithm}
    \textbf{Input:} We are given a graph $G = (V, E)$ in which each edge $e$ is associated with demand $q_e$ and length $d_e$, $W$ as the vehicle weight, and $k_{\text{max}}$ as the number of iterations. \\
    Call the Greedy constructive heuristic (Algorithm \ref{algo:greedy}) to initialize $\sigma^*$. \\
    $\sigma' \leftarrow \sigma^*$ \\
    \For{$k = 1 \rightarrow k_{\text{max}}$}{
        Perform $0.2 \times |E|$ random exchanges of edges within the sequence $\sigma'$. \\
        Apply local search on $\sigma'$ by using the operators 2-EXCHANGE, 1-OPT, and 2-OPT in the given order. We stop and apply the first improving operator in every step of the local search. \\
        \eIf{$z(\sigma') < z(\sigma^*)$}{
            $\sigma^* \leftarrow \sigma'$
        }{
            $\sigma' \leftarrow \sigma^*$
        }
    }
    \textbf{Output:} Return $\sigma^*$.
    \caption{Variable Neighborhood Search (VNS) metaheuristic \cite{332b87a03ada41f3b41ad5e1c0541165}}
    \label{algo:vns}
\end{algorithm}
\section{Attention model details}
\label{sec:appendix_attention_model_details}

\paragraph{Attention mechanism}
\label{sec:attention_mechanism}
The attention mechanism by \cite{vaswani2017attention} is interpreted as a weighted message passing algorithm between nodes in a graph as follows. The weight of the message \emph{value} that a node receives from a neighbor depends on the \emph{compatibility} of its \emph{query} with the \emph{key} of the neighbor. Formally, we define dimensions $d_{\text{k}}$ and $d_{\text{v}}$ and compute the key $\mathbf{k}_i \in \mathbb{R}^{d_{\text{k}}}$, value $\mathbf{v}_i \in \mathbb{R}^{d_{\text{v}}}$ and query $\mathbf{q}_i \in \mathbb{R}^{d_{\text{k}}}$ for each node by projecting the embedding $\mathbf{h}_i$:
\begin{equation}
\label{eq:enc_qkv}
	\mathbf{q}_i = W^Q \mathbf{h}_i, \quad \mathbf{k}_i = W^K \mathbf{h}_i, \quad \mathbf{v}_i = W^V \mathbf{h}_i.
\end{equation}
Here parameters $W^Q$ and $W^K$ are $(d_{\text{k}} \times d_{\text{h}})$ matrices and $W^V$ has size $(d_{\text{v}} \times d_{\text{h}})$.
From the queries and keys, we compute the compatibility $u_{ij} \in \mathbb{R}$ of the query $\mathbf{q}_i$ of node $i$ with the key $\mathbf{k}_j$ of node $j$ as the dot-product:
\begin{equation}
\label{eq:compatibility}
	u_{ij} = \begin{cases}
		\frac{\mathbf{q}_i^T \mathbf{k}_j}{\sqrt{d_{\text{k}}}} & \text{if $i$ adjacent to $j$} \\
        -\infty & \text{otherwise.}
    \end{cases}
\end{equation}
In a general graph, defining the compatibility of non-adjacent nodes as $-\infty$ prevents message passing between these nodes. From the compatibilities $u_{ij}$, we compute the \emph{attention weights} $a_{ij} \in [0, 1]$ using a softmax:
\begin{equation}
	\label{eq:attention_weights}
    a_{ij} = \frac{e^{u_{ij}}}{\sum_{j'}{e^{u_{ij'}}}}.
\end{equation}
Finally, the vector $\mathbf{h}_i'$ that is received by node $i$ is the convex combination of messages $\mathbf{v}_j$:
\begin{equation}
	\label{eq:readout}
    \mathbf{h}_i' = \sum_{j} a_{ij} \mathbf{v}_j.
\end{equation}

\paragraph{Multi-head attention}
According to \cite{vaswani2017attention}, it is beneficial to have multiple attention heads. This allows nodes to receive different types of messages from different neighbors. Especially, we compute the value in \eqref{eq:readout} $M = 8$ times with different parameters, using $d_{\text{k}} = d_{\text{v}} = \frac{d_{\text{h}}}{M} = 16$. We denote the result vectors by $\mathbf{h}_{im}'$ for $m \in \{1, \ldots, M\}$. These are projected back to a single $d_{\text{h}}$-dimensional vector using $(d_{\text{h}} \times d_{\text{v}})$ parameter matrices $W_m^O$. The final multi-head attention value for node $i$ is a function of $\mathbf{h}_1, \ldots, \mathbf{h}_n$ through $\mathbf{h}_{im}'$:
\begin{equation}
\label{eq:MHA}
	\text{MHA}_i(\mathbf{h}_1, \ldots, \mathbf{h}_n) = \sum\limits_{m=1}^{M} W_m^O \mathbf{h}_{im}'.
\end{equation}

\paragraph{Feed-forward sublayer}
The feed-forward sublayer computes node-wise projections using a hidden (sub)sublayer with dimension $d_{\text{ff}} = 512$ and a ReLu activation:
\begin{equation}
\label{eq:ff_layer}
	\text{FF}(\hat{\mathbf{h}}_i) = W^{\text{ff},1} \cdot \text{ReLu}(W^{\text{ff},0} \hat{\mathbf{h}}_i + \bm{b}^{\text{ff},0}) + \bm{b}^{\text{ff},1}.
\end{equation}

\paragraph{Batch normalization}
We use batch normalization with learnable $d_{\text{h}}$-dimensional affine parameters $\bm{w}^{\text{bn}}$ and $\bm{b}^{\text{bn}}$:
\begin{equation}
\label{eq:bn_layer}
	\text{BN}(\mathbf{h}_i) = \bm{w}^{\text{bn}} \odot \overline{\text{BN}}(\mathbf{h}_i) + \bm{b}^{\text{bn}}.
\end{equation}
Here $\odot$ denotes the element-wise product and $\overline{\text{BN}}$ refers to batch normalization without affine transformation.

\end{document}